\documentclass[letterpaper, 10 pt, conference]{ieeeconf}  

\IEEEoverridecommandlockouts                              

\usepackage{times}
\usepackage[fleqn]{amsmath}
\usepackage{xspace}
\usepackage{float}
\usepackage{balance}
\usepackage{amssymb,latexsym,amsfonts,amsmath}
\usepackage{xkeyval}
\usepackage{graphicx}
\include{diagrams}
\usepackage{eso-pic}
\usepackage{epsfig}
\usepackage{url}
\usepackage{color}
\usepackage{tikz}
\usetikzlibrary{automata}
\usetikzlibrary{shapes,arrows}
\usepackage[latin1]{inputenc}
\usepackage{dcolumn}
\usepackage{subfigure}
\usepackage{mathtools}
\usepackage[T1]{fontenc}
\usepackage{pgfplots}
\usepackage{mydefs}
\usepackage{graphicx}
\usepackage{algorithm}
\usepackage[algo2e,ruled,vlined,linesnumbered]{algorithm2e}
\SetAlFnt{\small\sffamily}
\usepackage[nodisplayskipstretch]{setspace}
\SetAlFnt{\small\rmfamily}

\usepackage[left=54pt,right=54pt,top=54pt,bottom=54pt]{geometry}
\usepackage{afterpage}

\usepackage{multicol}
\usepackage[bookmarks=true]{hyperref}

\newtheorem{theorem}{Theorem}[section]
\newtheorem{lemma}[theorem]{Lemma}

\newtheorem{problem}[theorem]{Problem}

\newtheorem{example}[theorem]{Example}

\numberwithin{equation}{section}
\usepackage{url}

\DeclareMathOperator*{\argmin}{argmin}

\title{\LARGE \bf
T* : A Heuristic Search Based Algorithm for Motion Planning with Temporal Goals
}

\author{Danish Khalidi$^{1}$, Dhaval Gujarathi$^{2}$ and Indranil Saha$^{3}$
\thanks{$^{1}$ Danish Khalidi is with NetApp India.
This work was carried out when Danish was an M.Tech student in the Department of Computer Science and Engineering,
Indian Institute of Technology Kanpur.
{\tt\small danish.khalidi08@gmail.com}}%
\thanks{$^{2}$ Dhaval Gujarathi is with SAP India.
This work was carried out when Dhaval was an M.Tech student in the Department of Computer Science and Engineering,
Indian Institute of Technology Kanpur.
{\tt\small dhavalsgujarathi@gmail.com}}%
\thanks{$^{2}$ Indranil Saha is with Department of Computer Science and Engineering,
Indian Institute of Technology Kanpur.
{\tt\small isaha@cse.iitk.ac.in}}%
}

\begin{document}

\maketitle
\thispagestyle{empty}
\pagestyle{empty}

\begin{abstract}
Motion planning is one of the core problems to solve for developing any application involving an autonomous mobile robot. The fundamental motion planning problem involves generating a trajectory for a robot for point-to-point navigation while avoiding obstacles. Heuristic-based search algorithms like A* have been shown to be  efficient in solving such planning problems. Recently, there has been an increased interest in specifying complex motion plans using temporal logic. 
In the state-of-the-art algorithm, the temporal logic motion planning problem is reduced to a graph search problem and Dijkstra's shortest path algorithm is used to compute the optimal trajectory satisfying the specification.

The A* algorithm when used with an appropriate heuristic for the distance from the destination can generate an optimal path in a graph more efficiently than Dijkstra's shortest path algorithm. The primary challenge for using A* algorithm in temporal logic path planning is that there is no notion of a single destination state for the robot. We present a novel motion planning algorithm T* that uses the A* search procedure  \emph{opportunistically} to generate an optimal trajectory satisfying a temporal logic query. Our experimental results demonstrate that T* achieves an order of magnitude improvement over the state-of-the-art algorithm to solve many temporal logic motion planning problems in 2-D as well as 3-D workspaces.
\end{abstract}

\section{Introduction} 

\label{sec-introduction} 

Motion planning is one of the core problems in robotics where we design algorithms to enable an autonomous robot to carry out a real-world complex task successfully~\cite{LaValle:2006:PA:1213331}. A basic motion planning task consists of point-to-point navigation while avoiding obstacles and satisfying some user-given constraints. To solve this problem, there exist many methods among which
graph search algorithms like A*~\cite{Astar4082128} and sampling based motion planning techniques such as rapidly exploring random trees~\cite{doi:10.1177/02783640122067453} are two prominent ones.


Recently, there has been an increased interest in specifying complex motion plans using temporal logic (e.g.~\cite{Kress-GazitFP07,KaramanF09,BhatiaKV10,WongpiromsarnTM12,ChenTB12,UlusoySDBR13, SahaRKPS14}).
Using temporal logic~\cite{Baier:2008:PMC:1373322}, one can specify requirements that involve temporal relationship between different operations performed by robots.
Such requirements arise in many robotic applications, including persistent surveillance~\cite{UlusoySDBR13}, assembly planning~\cite{Halperin1998}, 
evacuation~\cite{RodriguezA10}, search and rescue~\cite{Jennings97}, localization~\cite{Fox00}, object transportation~\cite{Rus95}, and formation control~\cite{Balch98}.


A number of algorithms exist for solving Linear Temporal logic (LTL) motion planning problems in different settings. For an exhaustive review on this topics, the readers are directed to the survey by Plaku and Karaman~\cite{PlakuK15}. In case of robots with continuous state space,
we resort to algorithms that do not focus on optimality of the path, rather they focus on reducing the computation time to find a path. For example, sampling based LTL motion planning algorithms compute a trajectory in continuous state space efficiently, but without any guarantee on optimality (e.g~\cite{sampling5509503},~\cite{Kantaros2017},~\cite{VasileB13}). 
On the other hand, the LTL motion planning problem where the dynamics of the robot is given in the form of a discrete transition system can be solved to find an optimal trajectory for the robot by employing graph search algorithms  (e.g.~\cite{belta5650896},~\cite{UlusoySDBR13}). 
In this paper, we focus on the class of LTL motion planning problems where a robot has discrete dynamics and seek to design a computationally efficient algorithm to generate an optimal trajectory for the robot.  

Traditionally, the LTL motion planning problem for the robots with discrete dynamics is reduced to the problem of finding the shortest path in a weighted graph and Dijkstra's shortest path algorithm is employed to generate an optimal trajectory satisfying an LTL query~\cite{belta5650896}. However, for a large workspaces and a complex LTL specification, this approach is merely scalable. Heuristics based search algorithms such as A*~\cite{Russell:2009:AIM:1671238} have been successfully used in solving point to point motion planning problems  and is proven to be significantly faster than Dijkstra's shortest path algorithm. 
The A* algorithm when used with an appropriate heuristic for distance from the destination node can generate an optimal path in a graph efficiently. A* algorithm have not been utilized in temporal logic path planning as there is no notion of a single destination state in LTL motion planning. We, for the first time, attempt to incorporate the A* search algorithm in LTL path planning to generate an optimal trajectory satisfying an LTL query efficiently. 

We have applied our algorithm to solving various LTL motion planning problems in 2-D and 3-D workspaces and compared the results with that of the algorithm presented in~\cite{belta5650896}.
Our experimental results demonstrate that T* in many cases achieves an order of magnitude better computation time than that of the traditional approach to solve LTL motion planning problems.


\section{Preliminaries} 
\label{sec-preliminaries}

 The inputs to our algorithm consists of the robot workspace $\mathcal{W}$, the set of robot actions $Act$, and a Linear Temporal Logic Query $\Phi$.

\subsection{Workspace, Robot Actions and Trajectory}
In this work, we assume that the robot operates in a two-dimensional (2-D) or a three dimensional (3-D) workspace $W$ which we represent as a grid map. The grid divides the workspace into square shaped cells. Each of these cells denote a state in the workspace $W$ which is referenced by its coordinates. Some cells in the grid could be marked as obstacles and in any case the robot is not allowed to visit such states. 

We capture the motion of a robot using  a set of Actions $Act$. The robot changes its state in the workspace by performing the actions in $Act$. An action $act\in Act$ is associated with a $cost$ which captures the energy consumption or time delay to execute it. A robot can move to satisfy a given specification by executing a sequence of actions in $Act$. The sequence of states followed by the robot is called the \emph{trajectory} of the robot. The \emph{cost of a trajectory} is defined as the sum of costs of the actions that are utilized to realize the trajectory. 

\subsection{Transition System}
We can model the motion of the robot in the workspace $\mathcal{W}$ as a weighted transition system. A weighted transition system for a robot  with the set of actions $Act$ is defined as $T := (S_T, s_0, E_T,\Pi_T, L_T, w_T, \mathcal{O}_T )$  where, (i) $S_T$ is the set of states/vertices, (ii) $s_0 \in S_T$ is the initial state of the robot, (iii) $E_T \subseteq S_T \times S_T$ is the set of transitions/edges, $(s_1 , s_2) \in E_T$  iff  $s_1 , s_2 \in S_T$ and $s_1 \xrightarrow[]{act} s_2$, where $act \in Act$, (iv) $\Pi_T$ is the set of atomic propositions, (v) $L_T : S_T \rightarrow 2^{\Pi_T}$ is a map which provides the set of atomic propositions satisfied at a state in $S_T$, and (vi) $w_T: E_T \rightarrow \mathbb{N}_{>0}$ is a weight function.
(vii) $\mathcal{O}_T$: set of obstacle cells in $\mathcal{W}$



We can think of a weighted transition system $T$ as a weighted directed graph $G_T$ with $S_T$ vertices,$E_T$ edges and $w_T$ weight function. Whenever we use some graph algorithm on a transition system T, we mean to apply it over $G_T$.  
\begin{example}
Throughout this paper, we will use the workspace $\mathcal{W}$ shown in Figure \ref{exampleGrid} for the illustrative example. We build transition system $T$ over $\mathcal{W}$ where $\Pi_T = \{P_1, P_2, P_3\}$. The proposition $P_i$ is $\mathtt{true}$ if the robot is in one of the locations denoted by $P_i$. From any cell in $\mathcal{W}$, robot can move to one of it's neighbouring four cells with cost 1. Cells with black colour represent obstacles($\mathcal{O}_T$) and in any scenario, robot cannot occupy them. 
\end{example}


\begin{figure} 
\centering
\subfigure[Transition system $T$ with propositions $P_{1},P_{2}\text{ and }  P_{3}$]{\centering \label{exampleGrid}\includegraphics[width=46mm]{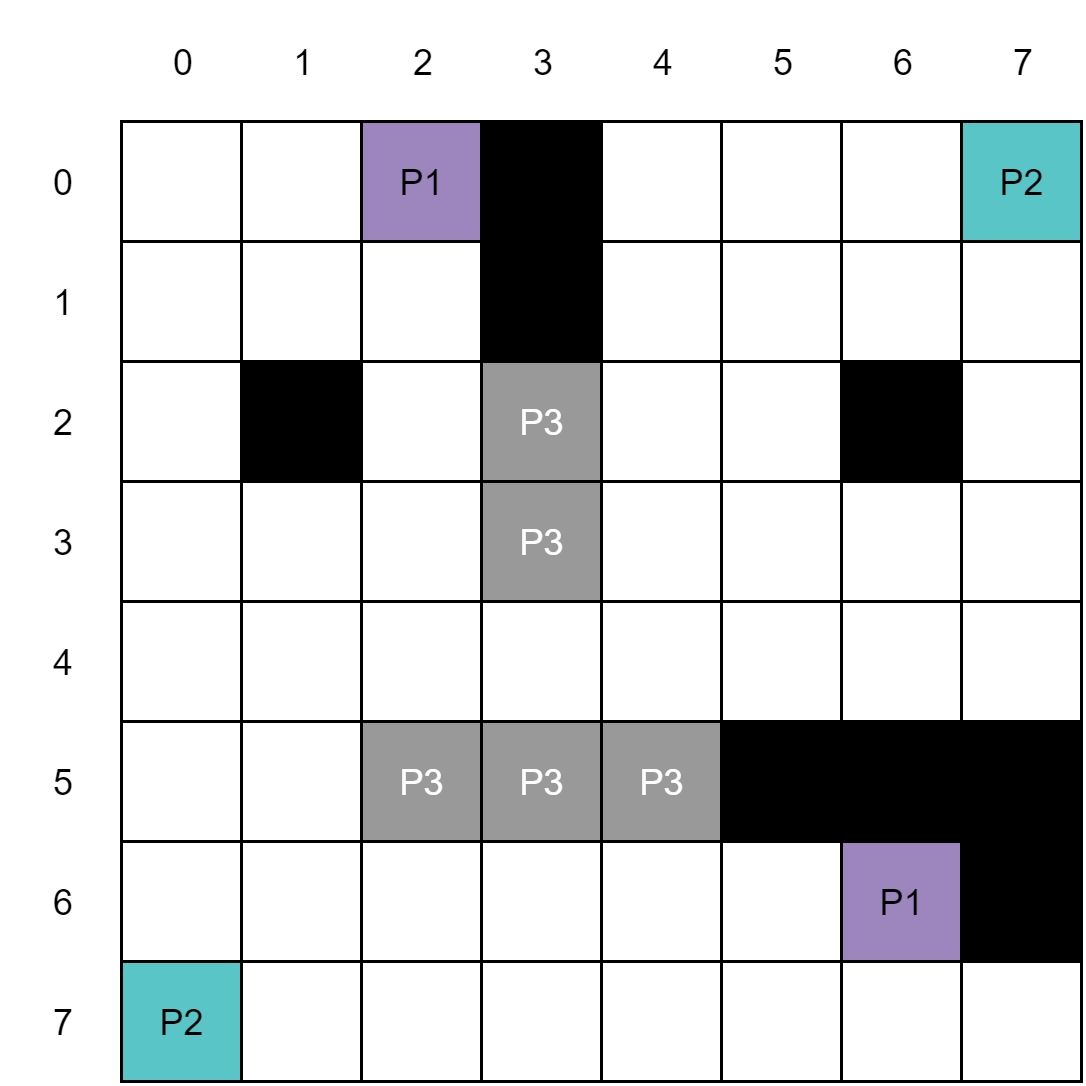}}
\hspace{1.2mm}
\subfigure[B\"{u}chi automaton B for query: $\square(\Diamond p_{1} \land \Diamond p_{2} \land \neg p_{3})$]{\centering \label{exampleAutomata}\includegraphics[width=37mm]{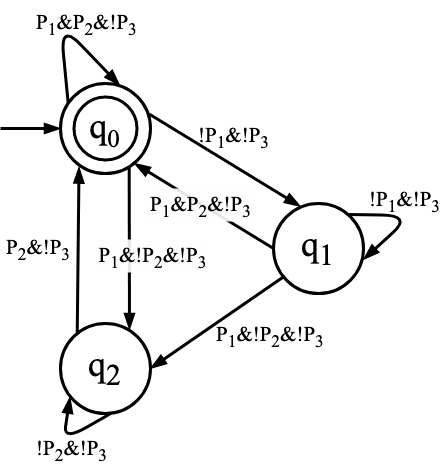}}
\caption{Transition System and B\"uchi Automaton}
\end{figure}

\subsection{Linear Temporal Logic}
The motion planning query/task in our work is given in terms of formulas written using \emph{Linear Temporal Logic} (LTL).
LTL formulae over the set of atomic propositions $\Pi_T$ are formed according to the following grammar~\cite{Baier:2008:PMC:1373322}:
$$\Phi:: = \mathtt{true} \| a \| \phi_{1} \land \phi_{2} \| \neg\phi \| X \phi \| \phi_{1} U \phi_{2}. $$ 

The basic ingredients of an LTL formulae are the atomic propositions $a \in \Pi_T$, the Boolean connectors like conjunction $\land$, and negation $\neg$, and two temporal operators $X$ (next) and $U$ (until). 
The semantics of an LTL formula is defined over an infinite trajectory $\sigma$. 
The trajectory $\sigma$ satisfies a formula $\xi$, if the first state of $\sigma$ satisfies $\xi$.
The logical operators conjunction $\land$ and negation $\neg$ have their usual meaning.
For an LTL formula $\phi$,
$X\phi$ is $\textsf{true}$ in a state if $\phi$ is satisfied at the next step. The formula $\phi_1 \ U \ \phi_2$ denotes that $\phi_1$ must remain $\textsf{true}$ until $\phi_2$ becomes $\textsf{true}$ at some point in future.
The other LTL operators that can be derived are  $\square$ (\emph{Always}) and $\Diamond$ (\emph{Eventually}). The formula $\square \phi$ denotes that the formula $\phi$ must be satisfied all the time in the future. The formula $\Diamond \phi$ denotes that the formula $\phi$ has to hold sometime in the future. We have denoted negation $\lnot P$ as $!P$ and conjunction as $\&$ in Figures. 




\subsection{B{\"u}chi Automaton}
For any LTL formulae $\Phi$ over a set of propositions $\Pi_T$, we can construct a B{\"u}chi automaton with input alphabet $\Pi_B = 2^{\Pi_T}$. We can define a B{\"u}chi automaton as ${B:= (Q_B, q_0, \Pi_B, \delta_B, Q_f)}$, where, (i) $Q_B$ is a finite set of states, (ii) $q_0 \in Q_B$ is the initial state, (iii) $\Pi_B = 2^{\Pi_T}$ is the set of input symbols accepted by the automaton, (iv) $\delta_B \subseteq Q_B \times \Pi_B \times Q_B$ is a transition relation, and (v) $Q_f \subseteq Q_B$ is a set of final states.
An accepting state in the B\"uchi automaton is the one that needs to occur infinitely often on an infinite length string consisting of symbols from $\Pi_B$ to get accepted by the automaton.
\begin{example}
Figure \ref{exampleAutomata} shows the B\"uchi automaton for an LTL task $\Phi = \square(\Diamond p_{1} \land \Diamond p_{2} \land \neg p_{3})$.  The state $q_0$ here denotes the start state as well as the final state. It informally depicts the steps to be followed in order to compete the task $\Phi$. The transitions $q_1 \rightarrow q_2 \rightarrow q_0$ lead to visit of state 
where $P_1 \land \lnot P_2 \land \lnot P_3$ is satisfied using only states which satisfy $\lnot P_1 \land \lnot P_3$ and then go to state where $P_2 \land \lnot P_3 $ using states which  satisfy $\lnot P_2 \land \lnot P_3$. 
We can understand the meaning of the other transitions from the context.  

\subsection{Product Automaton}\label{productAutomaton}
\sloppy{ The product automaton $P$ between the transition system $T = (S_T, s_0, E_T,\Pi_T, L_T, w_T, \mathcal{O}_T )$ and the B{\"u}chi automaton $B = (Q_B, q_0, \Pi_B, \delta_B, Q_f)$ is defined as $$P := (S_P, S_{P,0}, E_P, F_P, w_p)$$ 
where, (i) $S_P = S_T \times Q_B$, (ii) $S_{P,0}:= (s_0, q_0)$ is an initial state, (iii) $E_P \subseteq S_P \times S_P$, where $((s_i,q_k) ,(s_j,q_l)) \in E_P$ if and only if $(s_i, s_j)\in E_T$ and $\left(q_k, L_T\left(s_j\right),q_l\right) \in \delta_B$, (iv) the set of final states $F_P := S_T \times Q_f $, and $f(s_i, q_k) \in F_P $ if and only if $f \in S_P$ and $q_k \in Q_f$, and
(v) $w_P : E_P \rightarrow \mathbb{N}_{>0}$ such that $w_P((s_i,q_k) ,(s_j,q_l)) := w_T(s_i, s_j)$ for all $((s_i,q_k) ,(s_j,q_l)) \in E_P$.}
By it's definition, all the states and transitions in the product automaton follows LTL queries. You can refer \cite{belta5650896} for product graph examples.

\end{example}

\section{Problem Definition} 
\label{sec-problem}

\sloppy Consider a robot moving in a static workspace $(\mathcal{W})$ and its motion is modeled as a transition system $T=(S_T, s_0, E_T,\Pi_T, L_T, w_T, \mathcal{O}_T )$.
A run over the transition system $T$ starting at initial state $s_0$ and time instant $t_0$ defines the trajectory of the robot in $\mathcal{W}$.

Suppose, the robot has been given a task in the form of an LTL query $\phi$ over $\Pi_T$ which involves carrying out some tasks repetitively. 
As mentioned above, we can construct a B{\"u}chi automaton $B= (Q_B, q_0, \Pi_B, \delta_B, Q_f)$ from $\phi$. Let $\Pi_c = \{ c \:|\: c \in \Pi_B$ and $ \exists \delta_B(q_i , c) =  q_j$ where, $q_i \in Q_B$ and $q_j \in Q_f  \}$. 
The set $\Pi_c$ represents a set of formulae using which we can reach a final state in $B$ from some other state. Let $F_{\pi} = \{ s_i \: | \: s_i \in S_T$ and $ s_i \vDash \pi_j$ where $\pi_j \in \Pi_c \}$. $F_{\pi}$ represents the subset of the states of $T$ which satisfy the incoming transition to a final state in $B$.

Let us assume that there exists at least one run over $T$ which satisfies $\phi$. 
Let $\mathcal{R} = s_0,s_1,s_2,...$ be an infinite length run/path 
over $T$ which satisfies $\phi$ and the corresponding infinite length trace of sets of atomic propositions 
is $\mathcal{R}_\pi = \pi_0,\pi_1,\pi_2,...$ where $\pi_i = L_T(s_i)$.
Let $t_\mathcal{R} = t_0, t_1, t_2,...$ be the time sequence associated with $\mathcal{R}$, where $t_0 = 0$.
As $\mathcal{R}$ satisfies $\phi$, then there has to exist $f \in F_\pi$ occurring on $\mathcal{R}_\pi$ infinitely many times. From $\mathcal{R}_\pi$, we can extract all the time instances at which $f$ occurs. Let $t_\mathcal{R}^f(i)$ denotes the time instance of $i^{th}$ occurrence of state $f$ on $\mathcal{R}_\pi$.
Our goal is to synthesize an infinite run $\mathcal{R}$ which satisfies the LTL formulae $\phi$  and minimizes the cost function 
\begin{equation} \label{eq1}
\mathcal{C}(\mathcal{R})=\limsup_{i \rightarrow+\infty} \sum_{k=t_\mathcal{R}^f(i)}^{t_\mathcal{R}^f(i+1)-1} w_T(s_k, s_{k+1})
\end{equation}
The above cost function represents the maximum cost of a sub-run starting and ending at $f$.  With this cost function, we want to minimize the maximum time between the successive instances of $f$ which represents the completion of the task.
Now, we formally define the problem as follows:

\begin{problem}
Given a transition system $T$ capturing the motion of the robot in workspace $\mathcal{W}$ and an LTL formulae $\phi$ representing the task given to the robot, find an infinite length run $\mathcal{R}$ over $T$ which minimizes the cost function $\mathcal{C}(\mathcal{R})$ given by equation \ref{eq1}.
\end{problem}


\subsection{Prefix-Suffix Structure} The accepting run  $\mathcal{R}$ of infinite length can be divided into two parts namely  \emph{prefix} ($\mathcal{R}_{pre}$) and \emph{suffix} ($\mathcal{R}_{suf}$). A prefix is a finite run from initial state of the robot to an accepting state $f \in F_\pi$ and a suffix is a finite length run starting and ending at $f$ reached by the prefix, and containing no other occurrence of $f$. This suffix will be repeated periodically and infinitely many times to generate an infinite length run $\mathcal{R}$. So, we can represent run $\mathcal{R}$ as $\mathcal{R}_{pre}. \mathcal{R}_{suf}^\omega$, where $\omega$ denotes the suffix being repeated infinitely many times. 
\begin{lemma}\label{lemma1}
At least one of the accepting runs $\mathcal{R}$ that minimizes cost function $\mathcal{C}(\mathcal{R})$ is in prefix-suffix structure.


\begin{proof}
As we discussed earlier, there exists $f \in F_\pi$ occurring on $\mathcal{R}$ infinitely many times. Using this argument, we can prove this lemma as mentioned in \cite{belta5650896}
\end{proof}
\end{lemma}


Using Lemma \ref{lemma1}, we can say that if there exists a run $\mathcal{R}$ which  satisfies LTL formulae and minimizes \ref{eq1}, then there also exists a run $\mathcal{R}_\mathcal{C}$ which is in prefix-suffix structure, satisfies the LTL formulae and minimizes \ref{eq1}. The cost of such run is the cost of it's suffix. So, now our goal translates to determining an algorithm which finds minimum cost suffix run starting and ending at a state $f \in F_\pi$ and a finite length prefix run starting at initial state $s_0 \in S_T$ and ending at $f$. 
So, let $\mathcal{R} = \mathcal{R}_{pre}.\mathcal{R}_{suf}^\omega$, where $\mathcal{R}_{pre} = s_0,s_1,s_2,...,s_p$ be a prefix and $\mathcal{R}_{suf} = s_{p+1},s_{p+2},...,s_{p+r}$ be a suffix. We can redefine the cost functuion given in \ref{eq1} as
\begin{equation} \label{eq3}
 \mathcal{C}(\mathcal{R})= \mathcal{C}(\mathcal{R}_{suf}) = \sum_{i=p+1}^{p+r-1} w_T(s_i, s_{i+1})
\end{equation}
\begin{problem}
Given a transition system $T$ capturing the motion of the robot in workspace $\mathcal{W}$ and an LTL formulae $\phi$ representing the task given to the robot, find an infinite length run $\mathcal{R}$ over $T$ which minimizes the cost function $\mathcal{C}(\mathcal{R})$ given by equation \ref{eq3}.
\end{problem}

\subsection{Basic Solution Approach}

The basic solution to above problem uses the automata-theoretic model checking approach \cite{belta5650896}. Steps to it are outlined in the algorithm \ref{basicSolution}.


\begin{algorithm2e}
\SetAlgoLined
\DontPrintSemicolon
\setstretch{1.0}
\textbf{Input:} $T$ : a transition system, $\phi$ :LTL formulae\;
\textbf{Output: }A run $\mathcal{R}_T$ over $T$ that satisfies $\phi$\;
\BlankLine
Convert $\phi$ to a B\"uchi automaton $B$.\;
Compute the product automaton $P = T \times B$.\;
\For{all $f \in F_P$ }{
$\mathcal{R}_f^{suf} \gets \mathtt{Dijkstra's\_Algorithm}(\: \emph{f} \:,\:  \emph{f}\:)$\;
$\mathcal{R}_f^{pre} \gets \mathtt{Dijkstra's\_Algorithm}(\: S_{P,0} \:,\:  \emph{f}\:)$
}
$\mathcal{R}_P^{suf} \gets$ minimum of all $\mathcal{R}_f^{suf}$ \;
$\mathcal{R}_P^{pre} \gets$ prefix of $\mathcal{R}_P^{suf}$ \;
$\mathcal{R}_P = \mathcal{R}_P^{suf} . \mathcal{R}_P^{pre} $\;
Project $\mathcal{R}_P$ over $T$ to compute $\mathcal{R}_T$
\caption{Basic\_Solution}
\label{basicSolution}
\end{algorithm2e} 

The first step in this algorithm is to compute B\"uchi automaton from the given LTL query $\phi$. Next we compute the product automaton of the transition system $T$ and B\"uchi automaton $B$. In this product automaton, for each final state $f \in F_P$, we find a prefix run starting from initial state $S_{P,0}$ to $f \in F_P$ and then find minimum cost cycle starting and ending at $f$ using Dijkstra's algorithm. We then choose prefix-suffix pair with smallest $\mathcal{C}(\mathcal{R}_P)$ cost i.e. pair with minimum suffix cost and project it on $T$ to obtain the final solution. 


Let the total number of sub formulae in LTL formula $\Phi$ be denoted as $|\Phi|$. Thus, the maximum number of states in the automaton is $2^{|\Phi|}$. In this way, the LTL to B\"{u}chi automaton conversion has the computational complexity $\mathcal{O}(2^{|\Phi|})$ as mentioned in \cite{DuretLutz2004SPOTAE}.
We compute the product graph using the BFS algorithm. Thus, the complexity to compute the reduced graph is given as $\mathcal{O} (|S_P| + |E_P|)$. Since we run Dijkstra's algorithm for finding shortest suffix cycle for each point in $f \in F_P$, the complexity of do it is $\|F_P\| \times (\|E_{P}\| \times log \|S_{P}\|)$.
So, the overall complexity of the basic solution is $\mathcal{O}( 2 ^ {|\Phi|} + ( |S_P| + |E_P| ) + |F_P| * |E_P|*log|V_P| )$ . 

In the following section, we present T$^*$ algorithm that provides a significantly improved running time for generating an optimal trajectory satisfying a given LTL query.

\section{T$^*$ Algorithm}

The environments for robotic applications are continuous in nature. To reduce the amount of computation required during the motion planning of the robots in such environments, they are partitioned 
into smaller cells or blocks based on the dynamics of the robot. We call it a \emph{discrete workspace} in this paper. Heuristic information can be used to speed up the planning process in such workspaces \cite{Astar4082128}. To use the heuristic information, we need a concrete destination/goal to direct the search towards it. 
However, in the motion planning problem for LTL specifications, the robot may need to visit multiple locations repetitively in some order.
In most cases, there are multiple choices available to it to complete a task. Suppose in a pickup-drop application, there are multiple pick up and upload stations, and the robot has been asked to repeat the process of visiting any pick up station and then an upload station. In such a scenario, we cannot clearly specify a destination and direct path search towards it. T$^*$ attempts to use the heuristic information available in the discrete workspaces incrementally and thus achieves substantial speed up in terms of computation time over the basic solution \cite{belta5650896}. Also, as the size of the workspace increases, the size of the product automaton/graph also increases substantially, thus the time to run the Dijkstra's algorithm and memory consumption. T$^*$ does not compute the complete product graph, instead it computes a reduced version of the product graph which we call the reduced graph $G_r$, and thus achieves faster run times and lower memory consumption. Now, we introduce reduced graph $G_r$ that will be used in the T$^*$ algorithm.

\subsection{Reduced Graph}
Consider a product automaton $P$ of the transition system $T$ corresponding to the workspace shown in Figure~\ref{exampleGrid} and the B\"uchi automaton $B$ shown in Figure \ref{exampleAutomata}.  Suppose the initial location of the robot is $(4,7)$ and therefore the  initial state in the product automaton is $S_{P,0} = ((4,7), q_0)$. We know that the B\"uchi automata represents a task to be completed repeatedly and infinitely many times. Thus, from the initial state/location, the robot must move to reach a cycle
using the constraints from B\"uchi automaton (prefix) and then it must follow the cycle(suffix) completing the task repeatedly. 

To find such a path, we start from $((4,7), q_0)$ and move to state $((4,6), q_1)$ as per the definition of the product graph/automata in \ref{productAutomaton}. We are at the automata state $q_1$. To complete the task, we must visit a location where $P_1 \land \lnot P_2 \land \lnot P_3$ is satisfied so that we can move to B\"uchi state $q_2$ from $q_1$, and all the intermediate states till we reach such a state must satisfy $\lnot P_1 \land \lnot P_3$ formulae. Suppose next we move from $((4,6),q_1)$ to $((0,2), q_2)$ which satisfies $P_1 \land \lnot P_2 \land \lnot P_3$. There are many possible ways to go from $((4,6),q_1)$ to $((0,2),q_2)$ such as $((4,6), q_1) \rightarrow  ((4,5), q_1) \rightarrow ((4,4), q_1) \rightarrow ... \rightarrow ((1,2), q_1) \rightarrow ((0,2), q_2) $. 
To complete the cycle, the robot will next move to the state $((4,7),q_0)$. In the path from $((4,6),q_1)$ to $((0,2),q_2)$, we can observe the transition from B\"uchi automata state $q_1$ to $q_2$, 
and to achieve this, we ensure that the condition on the self loop over $q_1$ is satisfied by all the intermediate states.
Thus, using an analogy, we can consider the self loop transition condition $\lnot P_1 \land \lnot P_3$ over $q_1$ as the constraint which must be followed by the intermediate states while completing a task of reaching to the next state $q_2$ in the B\"uchi automaton by moving to the location which satisfies the transition condition from $q_1$ to $q_2$. Note that the self loop is the only means to navigate to next state. Using this as an abstraction method for the product automaton, we directly add an edge from state $((4,6),q_1)$ to $((0,2),q_2)$ in the abstract version of the product graph assuming that there exists a path between these two states and we will explore this path opportunistically only when it is required. This is the main idea behind the T$^*$ algorithm. We call this abstraction a \emph{reduced graph} which is constructed by removing transitions caused due to some self loops (not all self loops) resulting in a much smaller graph to work on. 

In this paper, we call atomic proposition with negation as a \emph{negative proposition} and an atomic proposition without negation as a \emph{positive proposition}. For example, $\lnot P2$ is a negative proposition and $P_2$ is a positive proposition. We divide the transition conditions in $B$ into two types. A transition condition which is a conjunction of all negative propositions is called a \emph{negative transition condition} and is denoted by $c_{neg}$. The one which is not negative is called \emph{positive transition condition}, and is denoted by $c_{pos}$. For example, $\lnot P_1 \land \lnot P_3 $ is a negative transition condition whereas $P_1 \land \lnot P_2 \land \lnot P_3$ is a positive transition condition. 

            
    

While constructing the reduced graph, we add edge from node $v_i(s_i, q_i)$ to  $v_j(s_j, q_j)$ as per following condition

\noindent\textbf{Condition: $\exists \delta_B(q_i, c_{neg}) = q_i$ and $\nexists \delta_B(q_i, c_{neg}) = q_j$} i.e. if there exists a negative self loop on $q_i$ and there does not exist any other negative transition from $q_i$ to some state in the B\"uchi automaton

\begin{enumerate}
    \item \textbf{\emph{If condition is true then}} add edge from $v_i$ to all $v_j$ such that $\exists \delta_B(q_i, c_{pos}) = q_j$ and $s_j \models c_{pos}$. Here, $q_i$ and $q_j$ can be same.  In short, in this condition we add all the nodes as neighbours which satisfy an outgoing $c_{pos}$ transition from $q_i$ and skip nodes which satisfy  $c_{neg}$ self loop transition assuming that $c_{neg}$ self loop transition can be used to find the actual path from $v_i$ to $v_j$ later in the algorithm. We add \emph{heuristic} cost as edge weight between $v_i$ and $v_j$. In this case, we call $v_j$ a \emph{distant neighbour} of node $v_i$ and henceforth we refer this condition as the \emph{distant neighbour condition}. 
    \item \textbf{\emph{If condition is  false then}}
    add edge from $v_i$ to all $v_j$ such that $\exists \delta_B(q_i, c) = q_j$ , $(s_i, s_j) \in E_T$ and $ s_j \models c $. This condition is same as per definition of the product automaton \ref{productAutomaton}. Here, as $s_i$ and $s_j$ are actual neighbours in the transition system, we add actual cost as edge weight between $v_i$ and $v_j$. Here also, $q_i$ and $q_j$ can be same. Here we add all the neighbours as per \emph{product automaton condition} for all the outgoing transitions from $q_i$.
\end{enumerate}
\subsubsection{Reduced Graph}\label{reducedGraph}
\sloppy{ We formally define \emph{Reduced Graph} for the transition system $T = (S_T, s_0, E_T,\Pi_T, L_T, w_T, \mathcal{O}_T )$ and the B\"uchi automaton $B = (Q_B, q_0, \Pi_B, \delta_B, Q_f)$  as 
$$G_r := (V_r, v_{0}, E_r, \mathcal{U}_r, F_r, w_r )$$ 

where, (i) $V_r \subseteq S_T \times Q_B$, set of vertices, 
(ii) $v_{0}:= (s_0, q_0)$ is an initial state, 
(iii) $E_r \subseteq V_r \times V_r$, is a set of edges added as per above condition 
(iv) $\mathcal{U}_r: E_r \rightarrow \{true,false\}$ a map which stores if the edge weight is a heuristic value or actual value
(v)$F_r \subseteq V_r \text{ and } v_i(s_i,q_i) \in F_r \text{ iff } q_i \in Q_B$. It is the set of final states. 
(vi) $w_r : E_r \rightarrow \mathbb{N}_{>0}$, weight function
}

\begin{figure}[H]
      \centering
      \includegraphics[scale=0.23]{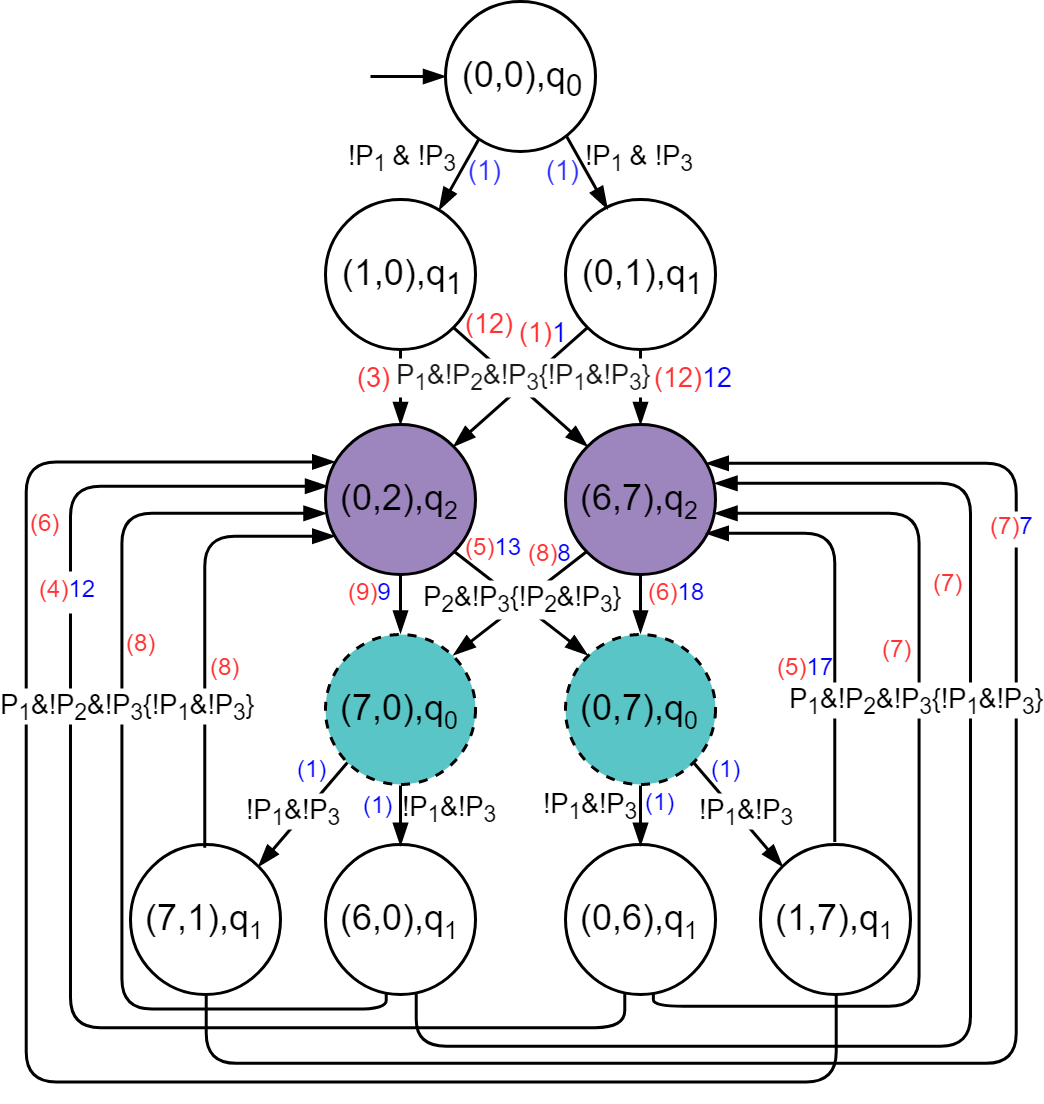}
      \caption{Reduced Graph for Transition System from Figure \ref{exampleGrid} and B\"uchi Automaton from Figure \ref{exampleAutomata}}
      \label{RedGraph}
\end{figure}

In procedure $\mathtt{Generate\_Redc\_Graph}$ of algorithm \ref{tStar} we run BFS algorithm starting from node $(s_0, q_0)$ and add the neighbours using the condition mentioned above. We use a map $\mathcal{U}_r$ to store the updated status of edges in $G_r$. $\mathcal{U}_r(v_i,v_j) = false$ says that weight of edge $(v_i, v_j)$ is heuristic cost between the two and we have not computed the actual cost between them.

The reduced graph obtained from the transition system $T$ in Figure~\ref{exampleGrid} and the B\"uchi automaton $B$ from Figure \ref{exampleAutomata} is shown in Figure \ref{RedGraph}. Edge weights in \emph{blue} color represents actual values whereas, in \emph{red} represents heuristic costs. Also, all the weights mentioned in round brackets '(--)' represents the state of reduced graph when it is constructed for the first time, whereas value beside the round brackets represents the actual value computed during procedure $\mathtt{Generate\_Redc\_Graph}$ later in the T$^*$ algorithm.

\begin{example}  
Suppose the robot is initially at location $(0,0)$ in the workspace $W$. We start with vertex $v_{0} = ((0,0),q_0)$. As $q_0$ does not have a self loop with negative transition condition, we add an edge from $S_{P,0}$ to $((1,0),q_1)$ as $((0,0), (1,0)) \in E_T$ and $(1,0) \models \lnot P_1 \land \lnot P_3$ as per product automaton condition. Similarly, we add an edge from $S_{P,0}$ to $((0,1),q_1)$. The cost of both of these transitions is 1 as these edges have been added with actual cost and not the heuristic cost.  
Next, we add neighbours of node $((1,0),q_1)$. As $q_1$ has a self loop with negative transition condition $\lnot P_1 \land \lnot P_3$, we add all the distant neighbours of $((1,0),q_1)$. We add an edge from $((1,0),q_1)$ to $((6,7),q_2)$ as $(6,7) \models (P_1 \land \lnot P_2 \land \lnot P_3)$. In Figure~$\ref{RedGraph}$, we represent this transition using notation $P_1 \land \lnot P_2 \land \lnot P_3\{ \lnot P_1 \land \lnot P_3 \}$ which says that $(6,7) \models (P_1 \land \lnot P_2 \land \lnot P_3)$ and all the intermediate nodes between the path from $(1,0)$ and $(6,7)$ will satisfy the condition $\lnot P_1 \land \lnot P_3$. As this node is added using distant neighbour condition, we update its edge weight with heuristic cost which is $11$ and shown using red color in the Figure \ref{RedGraph}. This way we keep on adding nodes to $G_r$. Figure~\ref{RedGraph} shows the complete reduced graph $G_r$.
\end{example}
  
\noindent Note: We have used Manhattan distance as heuristic cost. Always use the heuristic cost which is lower bound to  the actual cost.
 
\subsection{T$^*$ Procedure}
We have outlined the basic steps of T$^*$ in Algorithm~\ref{tStar}. The inputs to the algorithm ar the robot transition $T$ and robot task in the form of an LTL query $\phi$. The goal of the algorithm is to compute a minimum cost run $\mathcal{R}$ over $T$ which satisfies $\phi$. We first compute the B\"uchi automaton $B$ from the given LTL query $\phi$. Then we compute the reduced graph $G_r$ using $T$ and $B$.

\begin{algorithm2e}
\SetAlgoLined
\DontPrintSemicolon

\textbf{Input:} A transition system $T$, an LTL formulae $\phi$\;
\textbf{Output:} A minimum cost run $\mathcal{R}$ over $T$ that satisfies $\phi$\;

$B (Q_B, q_{0}, \Pi_B, \delta_B, Q_f) \gets \mathtt{ltl\_to\_Buchi}$
($\Phi$)\label{T* line 1}\;
$G_r(V_r, v_{0}, E_r, \mathcal{U}_r, F_r, w_r )\gets \mathtt{Generate\_Redc\_Graph(B,T)}$\label{T* line 2}\;
\
\For{all $f \in F_r$  \label{outerloop}}
{
    $N\gets 1$\;
    \While{N > 0}
    {
        $\mathcal{R}_f \gets \mathtt{Dijkstra\_Algorithm}
        (G_r, \:\text{\emph{f}} \:,\:\text{\emph{f}}\:)$\;
        $N \gets \mathtt{Update\_Edges}(\mathcal{R}_f, T, G_r, B)$\;
    }
    
    $\mathcal{R}_f^{suf} \gets \mathcal{R}_f$\;
    $v_0 \gets (s_0,q_0)$\;
    $\mathcal{R}_f^{pre} \gets \mathtt{Find\_Path} (G_r, \:v_0 \:,\:\emph{f}\:)$
    
}
$\mathcal{R}_P^{suf} \gets \argmin\limits_{\mathcal{R}_f^{suf} \text{ with a valid prefix}}  \mathcal{C}\left(\mathcal{R}_f^{suf}\right)$\;

$\mathcal{R}_P^{pre} \gets \mathtt{find\_prefix}(G_r, \mathcal{R}_P^{suf})$ \;

$\mathcal{R}_P \gets \mathcal{R}_P^{pre} . \mathcal{R}_P^{suf} $\;
project $\mathcal{R}_P$ over $T$ to compute $\mathcal{R}$\;
$\mathtt{return} \mathcal{R}$
\BlankLine
\BlankLine
\SetKwProg{myproc}{Procedure}{}{}\label{p1}
\myproc{ \textbf{Generate\_Redc\_Graph}($B, T$) } {

$v_{init} \gets v_0(s_0, q_0)$\;
    let $Q$ be a queue data-structure\;
    Initialize empty reduced graph $G_r$\;
    label $v_{init}$ as discovered and add it to $G_r$\;
    
    $Q.\mathtt{enqueue}(v_{init})$\;
    \While{$Q$ is not empty}
    {
        $v_i(s_i,q_l) \gets Q.\mathtt{dequeue}(\ )$\;
        \eIf{ $\exists \delta_B ( q_i, c_{neg} ) = q_i$ and $\nexists \delta_B ( q_i, c_{neg} ) = q_j$ }{
            \For{all $v_l(s_l,q_l)$ \text{such that} $\delta_B(q_i, c_{pos}) = q_l$ and $s_l \models c_{pos}$ }
            {
                $w_r(v_i, v_l) \gets \mathtt{heuristic\_cost}(s_i,s_l)$\;
                $\mathcal{U}_r(v_i, v_l) \gets false$\;
                \If{$v_l$ is not labelled as discovered}
                {
                    label $v_l$ as $discovered$ and add it to $G_r$\; 
                    $Q.\mathtt{enqueue}(v_l)$ 
                }
            }
        }
        {
            \For{all $v_l(s_l,q_l)$ such that $\delta_B(q_i, c) = q_l$, $ (s_i,s_j) \in E_T$ and $s_j \models c$ }
            {
                $w_r(v_i, v_l) \gets \mathtt{cost}(s_i,s_l)$\;
                $\mathcal{U}_r(v_i, v_l) \gets true$\;
                \If{$v_l$ is not labelled as discovered}
                {
                    label $v_l$ as $discovered$ and add it to $G_r$\; 
                    $Q.\mathtt{enqueue}(v_l)$ 
                }
            }
            
        }
    }
    return $G_r$
}

\BlankLine
\BlankLine
\SetKwProg{myproc}{Procedure:}{}{}\label{p2}
\myproc{ \textbf{Update\_Edges($\mathcal{R}_f, T, G_r, B$)} } {

$count \gets 0$  \;
    \For{ each edge $v_i(s_i,q_i) \rightarrow v_j(s_j,q_j)$ in $\mathcal{R}_f$}
    {
        \If{ $\exists \delta_B(q_i, c_{neg}) = q_i$ and $\mathcal{U}_r(v_i,v_j) = false$ }
        {
            $\mathcal{O}_T' \gets \{ s \:|\: s\in S_T \text{ and } s \models \neg c_{neg}  \}$\;
            $\mathcal{O} =\mathcal{O}_T \cup \mathcal{O}_T'$\; 
            $d \gets \mathtt{Astar}(T, \mathcal{O}, s_i, s_j)$\;
            $w_r(v_i, v_j) \gets d$, \ \ 
            $\mathcal{U}_r(v_i,v_j) \gets true$\;
            $count \gets count + 1$\;
        }
    }
    $\mathtt{return}$ $count$
}
\caption{T* Algorithm}
\label{tStar}
\end{algorithm2e}

Our aim is to find the minimum cost suffix cycle which starts and ends at a final state $f \in F_r$ in $G_r$. For each such state $f$, we find the minimum cost cycle/suffix run $\mathcal{R}_f$ starting and ending at $f$ using the Dijkstra's shortest path algorithm. 
Now, this cycle might have edges with heuristic cost, so in the next line, we update all the edges in $\mathcal{R}$ with actual cost using procedure $\mathtt{Update\_Edges}$. 
Only the edges which were added in $G_r$ using distant neighbour condition have heuristic value assigned to them and for such edges, all the intermediate nodes between two end points of edge must satisfy $c_{neg}$ self loop condition present on the B\"uchi state of the starting node. 
In this procedure, we first check if the B\"uchi automaton state of the starting node has a self loop with negative transition condition and if the edge has not been updated previously. Then, we find the path from starting node of the edge to the ending node in $T$ considering all the nodes which do not satisfy $c_{neg}$ as obstacles which are represented using $\mathcal{O}'_T$ in the procedure $\mathtt{Update\_Edges}$. Sub-procedure $\mathtt{Astar(s_i,s_j,T,\mathcal{O})}$ computes actual cost of the path from $s_i$ to $s_j$ in $G_T$ considering all the cells in $\mathcal{O}$ as obstacles. We do this step for each edge in $\mathcal{R}_f$ and return the number of edges updated in the run $\mathcal{R}_f$. If the number of edges updated in $\mathcal{R}_f$ is more than 0, then graph $G_r$ has been updated. As some the edge costs are different now, we again find the suffix run for $f$ and repeat the same procedure. But if the number of edges updated is 0 then we have found the minimum cost cycle starting and ending at $f$. We then find a prefix run from initial node $v_{0}(s_0,q_0)$ to $f$ using a procedure $\mathtt{Find\_Path}$. In this procedure, we first find a path from from initial node $(s_0,q_0)$ to $f$ in $G_r$ and then update all the edges as we did for suffix run except we do it just once. We have explicitly omitted the details of this procedure as it could be understood easily from the context. We then move on to the next final state and continue with the outer loop on line~\ref{outerloop}.

Once we find minimum cost suffix runs for all $f \in F_r$, we select the minimum cost suffix run among all $R_f^{suf}$ having a valid prefix as the final suffix $\mathcal{R}_P^{suf}$ and prefix $\mathcal{R}_P^{pre}$. We project it over $T$ to obtain the final satisfying run $\mathcal{R}$.

    

\begin{example}
We will continue with the example we have studied so far in this paper. From the transition system given in Figure \ref{exampleGrid} and the B\"uchi automaton for an LTL task $\square(\Diamond p_{1} \land \Diamond p_{2} \land \neg p_{3})$ given in Figure \ref{exampleAutomata}, we construct a reduced graph $G_r$ as shown in Figure \ref{RedGraph}. Now in $G_r$, we find the minimum cost suffix run/cycle containing a final state $f$ for each $f \in F_r$. Final states in Figure \ref{RedGraph} are shown using dotted circles. Let us start with $((0,7),q_0)$. We run Dijkstra's algorithm with $((0,7),q_0)$ as the source and the destination to obtain a cycle $\mathcal{R}_f = ((0,7),q_0) \xrightarrow[]{1} ((0,6),q_1) \xrightarrow[]{4} ((0,2),q_2) \xrightarrow[]{5} ((0,7),q_0)$. We run $\mathtt{Update\_Edges}$ over $\mathcal{R}_f$. The first edge  $((0,7),q_0) \xrightarrow[]{1} ((0,6),q_1)$ is already updated, no further modification is required. The second edge $((0,6),q_1) \xrightarrow[]{4} ((0,2),q_2)$ has not been updated before, and we compute the actual cost of this edge. There is a self loop with negative constraint $\lnot P_1 \land \lnot P_3$ over $q_1$ in $B$. So, we consider all the states which satisfy $\lnot(\lnot P_1 \land \lnot P_3) = P_1 \lor P_3$ as obstacles and find the shortest path from $(0,6)$ to $(0,2)$ in $T$. The actual cost of this path comes out to be $12$. So, we update the cost of edge from $((0,6),q_1)$ to $((0,2),q_2)$ to $12$ in $G_r$. Similarly, we update the cost of edge from $((0,2),q_2)$ to $((0,7),q_0)$ to $13$. We have updated 2 edges from current $\mathcal{R}_f$. So, we run Dijkstra's algorithm again in $G_r$ from $f$ to $f$ to compute the new cycle as $ \mathcal{R}_f = ((0,7),q_0) \xrightarrow[]{1} ((1,7),q_1) \xrightarrow[]{5} ((6,7),q_2) \xrightarrow[]{6} ((0,7),q_0)$. We again run $\mathtt{Update\_Edges}$ over $\mathcal{R}_f$ to obtain updated edges as $ \mathcal{R}_f = ((0,7),q_0) \xrightarrow[]{1} ((1,7),q_1)) \xrightarrow[]{17} ((6,7),q_2) \xrightarrow[]{18} ((0,7),q_0)$. In this iteration also, we updated two edges. So we run Dijkstra's algorithm again to obtain new $\mathcal{R}_f$ as $ \mathcal{R}_f = ((0,7),q_0)) \xrightarrow[]{1} ((0,6),q_1)) \xrightarrow[]{12} ((0,2),q_2) \xrightarrow[]{13} ((0,7),q_0)$. All the edges on this run are updated and so this is the minimum cost run containing the final state $((0,7),q_0)$. 
    
 
Now, we move to the second final state $f = ((7,0),q_0)$. We find the cycle as $ \mathcal{R}_f = ((7,0),q_0) \xrightarrow[]{1} ((7,1),q_1) \xrightarrow[]{7} ((6,7),q_2) \xrightarrow[]{8} ((7,0),q_0)$. In this, we find the actual cost of all the non-updated edges. Actual costs of all the non-updated edges in this run are same as heuristic costs. So, no edge weight is updated in $G_r$. So, this is our suffix run containing final  state $((7,0),q_0)$. 

We can observe here that, for this iteration, we only explored one cycle containing $((7,0),q_0)$. Using this algorithm saved us from exploring other possible cycles that start and end at $((7,0),q_0)$ but have higher or equal cost. We are able to obtain the solution without computing all the edges of the reduced graph. That is how $T^*$  solves the problem faster than the basic automata-theoretic model checking based approach. In Figure \ref{RedGraph}, we can observe the state of the reduced graph after all the iterations. All the edge weights in the red with round bracket and blue cost written beside them have been updated with actual cost during the algorithm, whereas edge weights which are written with red colour in round brackets and without blue color cost beside it were never explored. This shows the computations saved by the T$^*$ algorithm. 

From above cycles, we select cycle $((7,0),q_0) \xrightarrow[]{1} ((7,1),q_1) \xrightarrow[]{7} ((6,7),q_2) \xrightarrow[]{8} ((7,0),q_0)$ as $\mathcal{R}_P^{suf}$. We skip the computation of prefix $\mathcal{R}_P^{pre}$ as the procedure is almost same as that of computing the suffix run, except that we need one pass of Dijkstra's algorithm and $\mathtt{Update\_Edges}$ to compute it. 
Here, the prefix is $\mathcal{R}_P^{pre} = ((0,0),q_0) \xrightarrow{1} ((1,0),q_1) \xrightarrow{12} ((6,7),q_2) \xrightarrow{8} ((7,0),q_0)$. We project it over $T$ to obtain final solution as $\{(0,0) \rightarrow (1,0) \rightarrow (6,7) \rightarrow (7,0)  \}\{ (7,0) \rightarrow (7,1) \rightarrow (6,7) \rightarrow (7,0) \}^\omega$. We have only mentioned the abstract path here for convenience.
\end{example}
 
\subsection{Computational Complexity}
As we saw earlier, the LTL to B\"{u}chi automaton conversion has the computational complexity $\mathcal{O}(2^{|\Phi|})$ as mentioned in \cite{DuretLutz2004SPOTAE}.
We compute the reduced graph using the BFS algorithm. Thus, the complexity to compute the reduced graph is given as $\mathcal{O} (|V_r| + |E_r|)$. Let $S_{\phi}$ be the set of states of $T$ at which some proposition is defined and a state in $T$ has a constant number of neighbours. Thus, in the worst case, the number of nodes in the reduced graph $(|V_r|)$ is $\mathcal{O}(|S_{\phi}|)$.    
In T$^*$, we use Dijkstra's algorithm over $G_r$ to compute suffix run for each final state $f \in F_r$. During each run of Dijkstra's algorithm, we find a cycle and update all it's edges. At worst case, we might have to run Dijkstra's algorithm as many times as the number of edges in the reduced graph $G_r$. And in Update\_Edges algorithm, we compute the actual weight of the edge using A$^*$ algorithm over $T$. So, A$^*$ can also be invoked as many times as number of edges in $G_r$ in worst case and the complexity of each $A^*$ could be same as Dijkstra's algorithm in worst case which is  $\mathcal{O}(|S_T|* log (|E_T|))$. So, the overall computational complexity of the T$^*$ can be given as $\mathcal{O}( 2 ^ {|\Phi|} + ( |V_r| + |E_r| ) + |E_r| * |E_r|*log|V_r| + |E_r| *|S_T|* log |E_T| )$

\subsection{Correctness and Optimality}

We have proved the following two theorems to establish the correctness of our algorithm and the optimality of the trajectory generated by it. 

\begin{theorem}
The trajectory generated by the T* algorithm satisfies the given LTL query.
\end{theorem}
\begin{proof}
We generate the reduced graph using BFS algorithm starting at node $v_{0}(s_0,q_0)$. Suppose, we are exploring the neighbours of node $v_i(s_i,q_i)$ during BFS. if distant neighbour condition is satisfied, then we add non-neighbouring vertices into $G_r$ which satisfy an outgoing transition condition from $q_i$ in $B$ and in Update\_Edges we compute the actual path between such nodes such that all the intermediate nodes follow $c_{neg}$ transition condition. So, all the nodes on such paths follows the LTL formulae. And if we add using product automaton condition, then it also satisfies LTL formulae by it's definition. Hence, we can conclude that the overall trajectory generated using T$^*$ satisfies LTL formulae.
\end{proof}

\begin{theorem}
The algorithm computes the least cost suffix run starting and ending at accepting/final state in the product graph/automaton $G_{P}$.
\end{theorem}
\begin{proof}
In the whole algorithm, we only ignore self loop transitions with negative transition condition($c_{neg}$). So, we can conclude that all the final states present in the product graph will also be present in the reduced graph.  Next we argue that all the minimum cost paths starting and ending at a final state in product automata $P$ are preserved in the reduced graph $G_r$. We have only omitted negative self loop transitions in the reduced graph. Whenever a node has been added after skipping the negative self loop transition, path to it is being computed whenever required considering corresponding negative transition condition using the Dijkstra's algorithm. Also, we have initialised all the distant neighbour nodes using heuristic value which is lower bound to the actual cost. So, whenever we compute a run in the reduced graph using Dijkstra's algorithm, we update all the edges of it with actual cost. Now, all the other paths available in the reduced graph at that time will have costs less than or equal to their actual costs. So, if we computed a run using Dijkstra algorithm whose all the edges have been updated, it is indeed minimum cost run among all the possible runs. This argument concludes the proof        
\end{proof}

With these two theorems, we can establish the correctness and optimality of T$^*$ algorithm.

\section{Evaluation}
\label{sec-results}

In this section, we present several results to establish the computational efficiency of T* algorithm. The results have been obtained on a desktop computer with a 3.4 GHz quad core processor with 16 GB of RAM. We use \textsf{LTL2TGBA} tool~\cite{DuretLutz2004SPOTAE} as the LTL query to B\"{u}chi automaton converter.

\subsection{Workspace Description and LTL Queries}\label{ltl queries}
The robot workspace is represented as a 2-D or a 3-D grid. Each cell in the grid is denoted by integer coordinates ($x$,$y$) for a 2-D workspace, and by ($x$,$y$,$z$) for a 3-D workspace.
For the 2-D workspace, each of these cells have 2 horizontal, 2 vertical, and 4 diagonal neighbours. Similarly, for the 3-D workspace, there are 6 non-diagonal and 20 diagonal neighbours. The cost of an edge connecting diagonal neighbours is 1.5 and the non-diagonal neighbours is 1. We consider several data gathering tasks performed by a robot by visiting some locations. 

We evaluated T* algorithm for seven LTL queries borrowed from~\cite{SmithTBR11}.
The LTL queries are denoted by $\Phi_{A}, \Phi_{B}, \ldots, \Phi_{G}$. Here, we mention two of those LTL specifications, $\Phi_{C}$ and $\Phi_{D}$, in detail.
Here, propositions $p_1$, $p_2$, $p_3$ denote the data gathering locations and propositions $p_4$ and $p_5$ denote the data upload locations.

\noindent
1) We want the robot to gather data from all the three locations and upload the gathered data to one of the data upload locations. Moreover, after visiting an upload location, robot must not visit another upload location until it visits a data gathering location.
The query can be represented as  
$\Phi_{C} = \square( \Diamond p_1 \land \Diamond p_2 \land \Diamond p_3 ) \land \square(\Diamond p_4 \lor \Diamond p_5) \land \square( (p_4 \lor p_5) \to X((\neg p_4 \land \neg p_5) U ( p_1 \lor p_2 \lor p_3))).$

\noindent
2) It can happen that the data of each location has to be uploaded individually before moving to another gathering place. We can put an until constraint that once a data gathering location is visited, do not visit another gathering location until the data is uploaded. This can be captured as 
$\Phi_{D} = \Phi_{C} \land  \square( (p_1 \lor p_2 \lor p_3) \to X((\neg p_1 \land \neg p_2 \land \neg p_3) U ( p_4 \lor p_5))).$

\subsection{Results on Comparison with Standard Algorithm~\cite{belta5650896}}

We Compare T* algorithm with the standard Dijkstra's algorithm based LTL motion planning algorithm~\cite{belta5650896} on a couple of workspaces. 
The workspaces which have been borrowed from~\cite{SmithTBR11} and~\cite{SahaRKPS16} are shown in Figure~\ref{Plot for query C D in workspace 1} and Figure~\ref{Plot for query C D in workspace 2} respectively. The workspace is $100 \times 100$. The 3-D workspace size is $100 \times 100 \times 20$.
The co-ordinates of the pickup and drop locations have been chosen at obstacle-free locations. 
Each data point in Table~\ref{table-comparison} is an average of data obtained from 10 experiments.

\subsubsection{2-D Workspace}

\begin{table}[t]
\footnotesize
\begin{center}
\begin{tabular}{|c|c|c|c|c|c|c|}
\hline
Spec & \multicolumn{3}{c|}{Workspace 1} & \multicolumn{3}{c|}{Workspace 2} \\
\cline{2-7}
 & Baseline & T* & Speedup & Baseline & T* & Speedup\\
     \hline
     $\Phi_{A}$ & 1.09 & 0.28  & \FPeval{\result}{round(1.09/0.28,2)}%
    $\result$  & 1.265 & 0.877  & \FPeval{\result}{round(1.265/0.877,2)}%
    $\result$\\ 
     \hline
     $\Phi_{B}$ & 1.02 & 0.13  &  \FPeval{\result}{round(1.02/0.13,2)}%
    $\result$ & 1.167 & 0.429 & \FPeval{\result}{round(1.167/0.429,2)}%
    $\result$\\ 
     \hline
     $\Phi_{C}$ & 3.58 & 0.16  &  \FPeval{\result}{round(3.58/0.16,2)}%
    $\result$ & 4.215 & 0.350  & \FPeval{\result}{round(4.215/0.35,2)}%
    $\result$ \\ 
     \hline
    $\Phi_{D}$ & 4.93 & 0.27 & \FPeval{\result}{round(4.93/0.27,2)}%
    $\result$ & 5.979 & 0.549 & \FPeval{\result}{round(5.979/0.549,2)}%
    $\result$\\
    \hline
    $\Phi_{E}$ & 4.72  & 0.52  & \FPeval{\result}{round(4.72/0.52,2)}%
    $\result$ & 5.268  & 1.156 & \FPeval{\result}{round(5.268/1.156,2)}%
    $\result$\\
    \hline
    $\Phi_{F}$ & 9.57  & 0.29  & \FPeval{\result}{round(9.57/0.29,2)}%
    $\result$ & 11.276  & 0.557  & \FPeval{\result}{round(11.276/0.557,2)}%
    $\result$\\
    \hline
    $\Phi_{G}$ & 5.57 & 0.26  & \FPeval{\result}{round(5.57/0.26,2)}%
    $\result$ & 6.87 & 0.546  & \FPeval{\result}{round(6.87/0.54,2)}%
    $\result$\\
\hline
\end{tabular}
\end{center}
\caption{Comparison with Standard LTL Motion Planning Algorithm~\cite{belta5650896} for 2-D Workspace\label{table-comparison}}
\end{table}

Table~\ref{table-comparison} shows the speedup of T* over the standard algorithm for workspace 1 and workspace 2. From the table, it is evident that for both the workspaces and for several LTL queries, T* provides over an order of magnitude improvement in running time with respect to the standard algorithm.
Trajectories for queries $\Phi_{C}$ and $\Phi_{D}$ in workspace 1 and workspace 2 generated by T* algorithm are shown in Figure~\ref{Plot for query C D in workspace 1} and Figure~\ref{Plot for query C D in workspace 2} respectively.

\begin{figure}[H]
    \centering
    \includegraphics[scale=0.2]{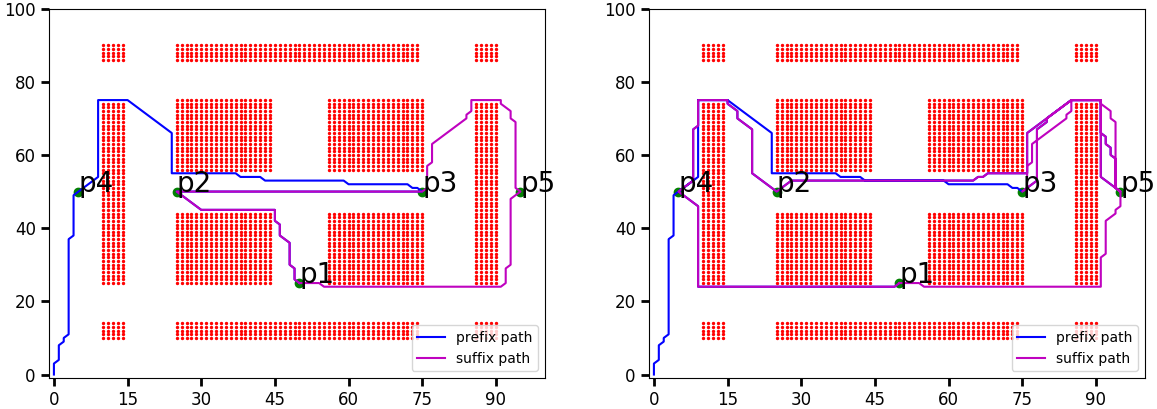}
    \caption{Trajectories for query $\Phi_{C}$ and $\Phi_{D}$ in workspace 1 generated by T*}
    \label{Plot for query C D in workspace 1}
\end{figure}

\begin{figure}[H]
    \centering
\includegraphics[scale=0.2]{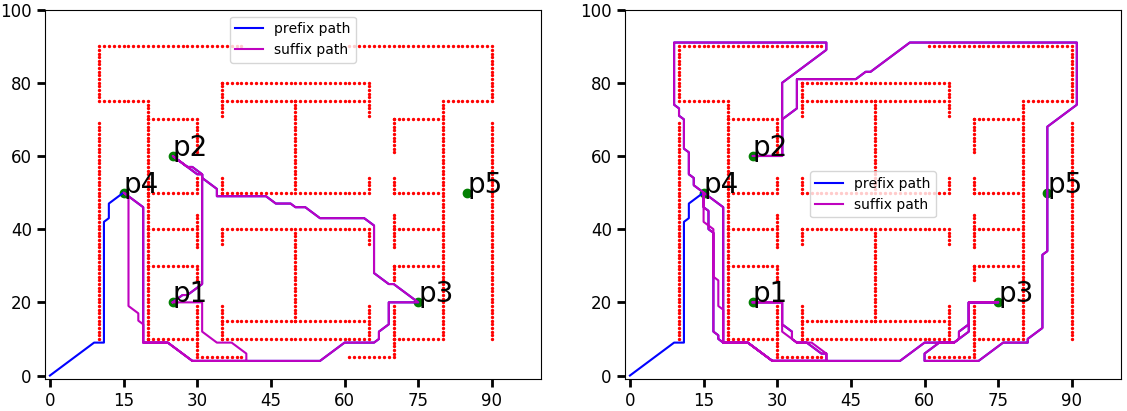}
    \caption{Trajectories for query $\Phi_{C}$ and $\Phi_{D}$ in workspace 2 generated by T*}
    \label{Plot for query C D in workspace 2}
\end{figure}

\subsubsection{3-D workspace}

\begin{table}[t]
\footnotesize
\begin{center}
\begin{tabular}{|c|c|c|c|c|c|c|}
\hline
Spec & \multicolumn{3}{c|}{Workspace 1} & \multicolumn{3}{c|}{Workspace 2} \\
\cline{2-7}
 & Baseline & T* & Speedup & Baseline & T* & Speedup\\
     \hline
     $\Phi_{A}$ & 105.88 & 36.10  & \FPeval{\result}{round(105.88/36.10,2)}%
    $\result$  & 84.32 & 24.42  & \FPeval{\result}{round(84.32/24.42,2)}%
    $\result$\\ 
     \hline
     $\Phi_{B}$ & 101.66 & 23.03  &  \FPeval{\result}{round(101.66/23.03,2)}%
    $\result$ & 81.57 & 19.04 & \FPeval{\result}{round(81.57/19.04,2)}%
    $\result$\\ 
     \hline
     $\Phi_{C}$ & 412.79 & 28.58  &  \FPeval{\result}{round(412.79/28.58,2)}%
    $\result$ & 519.55 & 22.91  & \FPeval{\result}{round(519.55/22.91,2)}%
    $\result$ \\ 
     \hline
    $\Phi_{D}$ & 464.69 & 51.19 & \FPeval{\result}{round(464.69/51.19,2)}%
    $\result$ & 497.99 & 35.89 & \FPeval{\result}{round(497.99/35.89,2)}%
    $\result$\\
    \hline
    $\Phi_{E}$ & 402.62 & 81.27  & \FPeval{\result}{round(402.62/81.27,2)}%
    $\result$ & 475.58  & 64.58 & \FPeval{\result}{round(475.58/64.58,2)}%
    $\result$\\
    \hline
    $\Phi_{F}$ & 869.98  & 47.01  & \FPeval{\result}{round(869.98/47.01,2)}%
    $\result$ & 941.00  & 35.86  & \FPeval{\result}{round(941.00/35.86,2)}%
    $\result$\\
    \hline
    $\Phi_{G}$ & 501.95 & 46.61  & \FPeval{\result}{round(501.95/46.61,2)}%
    $\result$ & 507.38 & 46.24  & \FPeval{\result}{round(507.38/46.24,2)}%
    $\result$\\
\hline
\end{tabular}
\end{center}
\caption{Comparison with Standard LTL Motion Planning Algorithm~\cite{belta5650896} over 3-D Workspace\label{table-comparison-3d}}
\end{table} 
Table~\ref{table-comparison-3d} shows the speedup of T* over the standard algorithm for 3-D workspace.

\subsubsection{Memory Consumption Comparison}
Table \ref{memoryComparison} compares the memory used by the both the algorithms when we scale the workspace from Figure \ref{Plot for query C D in workspace 1}  keeping other parameters constant. With the increase in size, the size of the product automaton increases, but the size of the reduced graph remains same. After $500 \times 500$, memory required to run A$^*$ dominates and hence memory consumption of T$^*$ also start increasing slowly.

\begin{table}[t]
\footnotesize
\begin{center}
\begin{tabular}{|c|c|c|c|c|}
\hline
     Workspace Size & Spec  & Baseline(KB) & T*(KB) & $\%$ Savings \\
     \hline
     $100 \times 100$ & $\Phi_D$  & 42.7 & 18.8 &  \FPeval{\result}{round((42.7 - 18.8)/42.7*100,1)}$\result$\\
     \hline
     $200 \times 200$ & $\Phi_D$  & 167.3  & 18.5 & \FPeval{\result}{round((167.3 - 18.5)/167.3*100,1)}$\result$  \\
     \hline
     $300 \times 300$ & $\Phi_D$  & 375.7 & 18.5 & \FPeval{\result}{round((375.7 - 18.5)/375.7*100,1)}$\result$ \\
     \hline
     $400 \times 400$ & $\Phi_D$  & 671.7 & 18.5 & \FPeval{\result}{round((671.7 - 18.5)/671.7*100,1)}$\result$ \\
     \hline
     $500 \times 500$ & $\Phi_D$  & 1072.38 & 25.5 & \FPeval{\result}{round((1072.38 - 25.5)/1072.38*100,1)}$\result$\\
     \hline
     $600 \times 600$ & $\Phi_D$ & 1510 & 34.34 &  \FPeval{\result}{round((1510 - 34.34)/1510*100,1)}$\result$\\
\hline
\end{tabular}
\end{center}
\caption{Memory Usage Comparison with Baseline Solution~\cite{belta5650896} }
\label{memoryComparison}
\vspace{-0.3cm}
\end{table}

\subsection{Analysis of T* Performance with Different Parameters}
\label{Analysis of T* with parameters}
This section contains the results related to the speedup of T* in comparison to the standard algorithm with the change in obstacle density, size of the workspace, and complexity of the LTL queries. 
In these results, we exclude the time taken for LTL to B\"{u}chi automaton conversion as this step requires equal amount of time for both the algorithms.

\begin{figure}[t]
\
\centering
{
\subfigure[Obstacle density]{
\centering
\resizebox{0.3\linewidth}{!}
{
\begin{tikzpicture}
\begin{axis}[
	x tick label style={
		/pgf/number format/1000 sep=},
	ylabel=\Huge{speedup},
	xlabel=\Huge{obstacle density},
	enlargelimits=0.05,
	legend style={at={(0.5,-0.2)},
	anchor=north,legend columns=-1},
]
\addplot
coordinates { (5, 91.1009423854)
(10, 92.3964207221)
(15, 76.5106748609)
(20, 58.7999975124)
(25, 49.3844380337)
(30, 37.739259902)
(35, 29.5017754947)
(40, 24.4415669048)};
\end{axis}
\end{tikzpicture}
}}
\centering
\subfigure[Spec complexity]{
\resizebox{0.3\linewidth}{!}
{
\begin{tikzpicture}
\begin{axis}[
	x tick label style={
		/pgf/number format/1000 sep=},
	ylabel=\Huge{speedup},
	xlabel= \Huge{number of locations},
	enlargelimits=0.05,
	legend style={at={(0.5,-0.2)},
	anchor=north,legend columns=-1},
]
\addplot
    coordinates { (3, 17.3691411351)
(6, 37.3735442692)
(9, 62.6985502942)
(12, 108.551276015)};
\end{axis}
\end{tikzpicture}
}}
\subfigure[Workspace size]{
\resizebox{0.3\linewidth}{!}
{
\begin{tikzpicture}
\begin{axis}[
	x tick label style={
		/pgf/number format/1000 sep=},
	ylabel=\Huge{speedup},
	xlabel= \Huge{x-axis dimension},
	enlargelimits=0.05,
	legend style={at={(0.5,-0.2)},
	anchor=north,legend columns=-1},
]
\addplot
coordinates { 
(20, 11.4544874973)
(100, 19.8989457192)
(200, 21.722599786)
(300, 22.4555525637)
(400, 22.4854421654)
(500, 22.149291206)
};
\end{axis}
\end{tikzpicture}
}}
}
\caption{Speedup achieved by T* }
\label{fig:optimizationGraph-ncs}
\vspace{-0.5cm}
\end{figure}
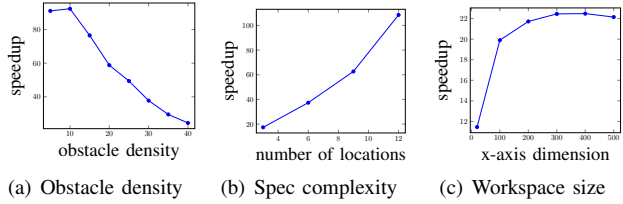

\subsubsection{Obstacle Density}
On increasing the obstacle density from 5 to 40 percent in a 2-D workspace of size $100\times100$, the speedup of T* in comparison with the standard algorithm for LTL query $\Phi_{D}$ decreases as shown in Figure~\ref{fig:optimizationGraph-ncs}(a). The obstacle locations are generated randomly.
Due to the increase in the obstacle density, the heuristic distances become significantly lesser than the actual distances. This results in an increase in the number of times the Dijkstra's algorithm invoked during the computation of the $\mathcal{R}_f^{suf}$ and updates to edge costs in $G_{r}$ as many times we find lower cost suffix run whose cost increases by a lot after Update\_Edges algorithm. This causes the reduction in the performance of T$^*$. As T$^*$ is a heuristic based algorithm, lesser the difference between heuristic cost and the actual cost, higher is the performance.
\subsubsection{Complexity of LTL Query}
We consider the LTL query $\Phi_{D}$ for this experiment.
Starting with 2 gather and 1 upload locations, the number of gather locations is incremented by 2 and that of the upload locations by 1 for 4 instances. The speedup is as shown in Figure~\ref{fig:optimizationGraph-ncs}(b). Speedup increases as $T^*$ explores available choices opportunistically based on the heuristic values whereas baseline solution explores all the choices gradually. speedup may vary based on the problem instance as this is heuristic based algorithm.  

\subsubsection{The size of the workspace}
We experimented with query $\Phi_{D}$ on Workspace 1 shown in Figure~\ref{Plot for query C D in workspace 1}
by increasing the size of the square grid. 
We considered the following sizes for the side length of the square grid: 20, 100, 200, 300, 400 and 500.
As shown in Figure~\ref{fig:optimizationGraph-ncs}, Speedup increases with increase in the size of the workspace keeping the other parameters constant. Size of the product graph increases significantly and hence time to tun the Dijkstra's algorithm. Whereas the size of the reduced graph remains same but time to run A$^*$ will increase with increase in workspace size.

\subsection{Experiments with Robot}
\label{Experiment with robot}
We used the trajectory generated by T* algorithm to carry out experiments with a Turtlebot on a 2-D grid of size $5\times5$ with four non-diagonal movements to the left, right, forward and backward direction. 
The cost of the forward and backward movement is 1,
whereas the cost of the left and righ movemnet is 1.5 as it involves a rotation followed by a forward movement.
The trajectories corresponding to the two queries $\Phi_{C}$ and $\Phi_{D}$ were executed by the Turtlebot. The location of Turtlebot in the workspace was tracked using Vicon localization system~\cite{vicon}. The video of our experiments is submitted as a suplimentary material.

\section{Conclusion}
\label{sec-conclusion}

In this work, we have developed a static LTL motion planning algorithm for robots with transition system with discrete state-space. 
Our algorithm opportunistically utilizes A* search which expands less number of nodes and thus is significantly  faster than the standard LTL motion planning algorithm based on Dijkstra's shortest path algorithm. Our future work includes evaluating our algorithm for non-holonomic robotic systems and extending it for multi-robot systems and dynamic environments.

\balance{

}

\end{document}